\documentclass{article} 
\usepackage{nips14submit_e,times}
\usepackage{hyperref}
\usepackage{url}

\usepackage{amsfonts}
\usepackage{textcomp}
\usepackage{amssymb}
\usepackage{marvosym}
\usepackage{amsmath}
\usepackage{verbatim}
\usepackage{pifont}
\usepackage{graphicx}
\usepackage{algorithm}
\usepackage{algorithmic}
\usepackage{amsthm}

\usepackage[raggedright]{sidecap}      

\usepackage{hyperref}	
\hypersetup{
  colorlinks,
  citecolor=blue,
  linkcolor=blue,
  urlcolor=blue}
 \usepackage[comma,authoryear]{natbib}

\usepackage{hongweiCommonMacro}
\usepackage{hwCrowdHeader}

\title{Cheaper and Better:\\ Selecting Good Workers for Crowdsourcing}

\author{Initially drafted by Hongwei Li and Qiang Liu}
\author{
Hongwei Li\\
Department of Statistics\\
Univeristy of California, Berkeley\\
\texttt{hwli@stat.berkeley.edu} \\
\And
Qiang Liu\\
Department of Computer Science \\
Univeristy of California, Irvine\\
\texttt{qliu1@uci.edu} \\
}

\nipsfinalcopy 

\begin{document}

\maketitle

\begin{abstract}
Crowdsourcing provides a popular paradigm for data collection at scale.
We study the problem of selecting subsets of workers from a given worker pool to maximize the accuracy under a budget constraint.   
One natural question is whether we should hire as many workers as the budget allows, or restrict on a small number of  top-quality workers. 
By theoretically analyzing the error rate of a typical setting in crowdsourcing, we frame the worker selection problem into a combinatorial optimization problem and propose an  algorithm to solve it efficiently.  
Empirical results on both simulated and real-world datasets show that our algorithm is able to select a small number of high-quality workers, and performs as good as, sometimes even better than, the much larger crowds as the budget allows. 

\end{abstract}



\section{Introduction}\label{sec:intro}

The recent rise of the crowdsourcing approach has made it possible to collect large amounts of human-labeled data and solve challenging problems that require human intervention at a large scale and at a relatively low cost. In micro-task marketplaces such as Amazon Mechanical Turk, the requestors can hire large numbers of online crowd workers to complete human intelligence tasks (HITs) in a short time and with payment as low as several cents per task. 
Unfortunately, because of the low pay and inexperience of the workers, their labeling qualities are often much lower than those of experts. 
A common solution is to add redundancy, asking many crowd workers to answer the same questions, 
and aggregating their answers; the combined results of the crowds are often  much better than that of an individual worker, sometimes even as good as that of the experts -- a phenomenon known as \emph{wisdom of crowds}. 

However, because the crowd workers often have different reliabilities due to their diverse backgrounds, it is important to weight their answers properly when aggregating their answers. 
A large body of work has been proposed to deal with the uncertainty and diversity on the workers' reliabilities;  
these methods often have a form of weighted majority voting where the answers of the majority of the workers are selected, with a weighting scheme that accounts the importance of the different workers according to their reliabilities. 
The workers' reliabilities can be estimated either using gold standard questions with known answers \citep[e.g.,][]{von2008recaptcha, Liu2013score}, 
or by statistical methods such as Expectation-Maximization (EM) \citep[see, e.g.,][]{Dawid79jrss, Whitehill2009nips, Karger2011nips, Liu2012, Zhou2012}. 


Our work is motivated by a natural question: 
do more crowd workers necessarily yield better aggregated results than less workers? 
The idea of \emph{wisdom of crowds} seems to suggest a confirmative answer, since ``\emph{larger crowds} should be \emph{wiser}". 
From a Bayesian perspective, this would be true if we had perfect knowledge about the workers' prediction model, and 
we were able to use an oracle aggregation procedure that performs exact Bayesian inference. 
However, in practice, because the workers' prediction model and reliabilities 
are never known perfectly,  
we run the risk of adding noisy information as we increase the number of workers. In the extreme, there may exist a large number of ``spammers", who submit completely random answers rather than good-faith attempts to label; 
adding these spammers would certatinly deteriorate the results, unless we are able to identify them perfectly, and assign them with zero-weights in the label aggregation algorithm. 
Even if there exist no extreme spammers, the median-level workers may still decrease the overall accuracy if they dominate over the small number of high-quality workers. 
In fact, a recent empirical study \citep{soll2013} shows that the aggregated results of a small number of (3 to 6) high-quality workers are often more accurate than those of much larger crowds. 

In this work, we study this phenomenon by formulating a worker selection problem under a budget constraint.  
Assume we have a pool of workers whose reliabilities have been tested by a small number of gold standard questions; under certain label aggregation algorithm, we want to select a subset of workers that maximizes the accuracy, with a budget constraint that the number of workers assigned {per task} is no more than $\Kmax$. 
A na{\"i}ve and commonly used procedure is to simply select the top $\Kmax$ workers that have the highest reliabilities. 
However, 
due to the noisy nature of the label aggregation algorithms (e.g., majority voting or EM), 
 selecting all the $\Kmax$ workers does not necessarily give the best accuracy, and may cause a waste of the resource. 
 We study this problem under a simple  label aggregation algorithm based on weighted majority voting, 
and propose a worker selection method that is able to select fewer ($\leq \Kmax$) top-ranked workers, while achieve almost the same, or even better aggregated solutions than the na{\"i}ve method that uses more (all the top $K$) workers. 

Our method is derived by framing the problem into a combinatorial optimization that minimizes an upper bound of the error rate, and deriving a globally optimal algorithm that selects  a group of top-ranked workers that optimize the upper bound of the error rate. 
We demonstrate the efficiency of our algorithm by comprehensive experiments on a number of real-world datasets.

{\bf Related work.}
There are many literatures on estimating the workers' reliabilities and eliminating the spammers based on a predefined threshold \citep[see e.g.,][]{Raykar2012eliminate, Joglekar2013}. 
Our work instead focuses on selecting a minimum number of highest-ranked workers while discarding the others (which are not necessarily spammers). 
Note that our method has the advantage of requiring no pre-specified threshold parameters.
Our work also should be distinguished with another line of research on online assignment for crowdsouring \citep[][ etc.]{Chen2013, Ho2013}, which have different objectives and purposes from our work.    

\textbf{Outline.} 
The rest of the paper is organized as follows.
We introduce the background and the problem setting in Section~\ref{sec:background}. 
We then formulate the worker selection problem into a combinatorial optimization problem and derive our algorithm in Section~\ref{sec:workerSelect}. The numerical experiments are presented in Section~\ref{sec:expResults}. 
We give further discussions in Section \ref{sec:discuss} and conclude the paper in Section~\ref{sec:conclusion}.

\section{Background and problem setting}
\label{sec:background}

Assume there are $M$ crowd workers and $N$ items (or questions) each with labels from $L$ classes. For notation convenience, we denote the set of workers by $\workset = \M$, the set of items  by $\N$ and the set of label classes by $\Labset$, where we use $\M$ to denote the set of first $M$ integers. We assume each item $j$ is associated with an unknown true label $\yj\in \Labset$, $j\in\N$. 
We also assume that we have $n$ \emph{control} (or \emph{gold standard}) questions whose true labels $y_j' \in\Labset, ~j\in[n]$ are known. 

When item $j$ is assigned to worker $i$ for labeling, we get a possibly inaccuracy answer from the worker, which we denote by $\zij\in\Labset$.  
The workers often have different expertise and attitude, and hence have different reliabilities. 
We assume $i$-th worker labels the items correctly with probability $\wi$, that is, $\wi =\P(\zij= \yj)$.  
%
%
%
In addition, assume we have an estimation of the workers' reliability $\hwi$, which can be estimated either based on the workers' performance on the control items, or by probabilistic inference algorithms like expectation-maximization (EM). 
With a known reliability estimation $\hwi$, most label aggregation algorithms, including the na{\"i}ve majority voting and EM, can be written into a form of weighted majority voting, 
\begin{align}
\label{equ:wmvGeneral}
\hat y_j = \argmax_{k \in \Labset}  \sum_{i\in\S}  \wf(\hwi) \cdot \I{\zij=k},  
\end{align}
where $\wf(\hwi)$ is a monotonic weighting function that decides how much the answers of worker $i$ contribute to the voting according to the reliability $\hwi$,  
and $\I{\cdot}$ is the indicate function. For majority voting, we have $\wf_{\text{mv}}(\hwi)= 1$, which ignores the diversity of the workers and may performance badly in practice. 
In contrast, a log-odds weighting function $\wflog(\hwi)= \mathrm{logit}(\hwi) - \mathrm{logit}(1/\nL)$, where $\mathrm{logit}(\hwi)\overset{def}{=} \log (\frac{\hwi}{1-\hwi})$, can be derived using Bayesian rule under a simple model that assumes uniform error across classes; here $1/\nL$ is the probability of random guessing among $L$ classes. 
However, in practice, the log-odds may be over confident (growing to infinite) when $\hat w_i$ is close to 1 or 0. 
A linearized version $\wflin(\hwi) = \hwi - 1/\nL$ has better stability, and is simpler for theoretical analysis \citep{Li2013b}.

Note that both of $\wflog$ and $\wflin$ have properties that are desirable for general weighting functions:  
Both are monotonic increasing functions of $\hwi$, and take zero value if $\hwi  = 1/\nL$ (to exclude the labels from random guessers); they are both positive if $\hwi > 1/\nL$ (better than random guessers), and are both negative if $\hwi <1/\nL$ (worse than random guessers).  
These common properties make $\wflin$ and $\wflog$ work similarly in practice.  But since $\wflin$ is more stable and simpler for theoretical analysis, we will focus on the linear weighted function $\wflin$ for our further development on the worker selection problem, that is, the labels are aggregated via (referred as \emph{WMV-linear}), 
\begin{eqnarray}\label{def:wmv_linear}
\hat y_j = \argmax_{k \in \Labset}  \sum_{i\in\S}  (\nL \hwi -1) \cdot \I{\zij=k}. 
\end{eqnarray}
In the next section, we study the worker selection problem and propose an efficient algorithm based on the analysis of the WMV-linear aggregation method.

\section{Worker selection by combinatorial optimization}\label{sec:workerSelect}

The problem of selecting an optimal set of workers requires predicting the error rate with a given worker set, which is unfortunately intractable in general. However, it is convenient to obtain an upper bound of the error rate for the linear weighted majority voting. 

\begin{thm}\label{thm:erBound}
Given a set $S$ of workers, using the weighted majority voting in \eqref{def:wmv_linear} with linear weights $\wflin$ and an unbiased estimator of the reliabilities $\hua{\hwi}_{i\in\S}$ that satisfies $\E[\hwi] = \wi$. If the workers' labels are generated independently according the following probability
\begin{eqnarray}\label{model}
\P(\zij=l|\yj=k)= 
\begin{cases}
\wi  &\text{ if~~} l=k, \\
\frac{1-\wi}{\nL-1} &\text{ if~~} l\neq k.
\end{cases}
\end{eqnarray}
Then we have
\begin{align}
&\inv{N}\sumj \P(\hyj\neq \yj)\leq \exp\left[ {-  \frac{2{\FS}^2}{\nL^2(\nL-1)^2} + \ln(\nL-1)} \right],\label{ineq:erBound}\\
&\text{where~~~~~~~~~~} 
\FS = \inv{\sqrt{|\S|}} \sum_{i\in\S} (\nL \wi - 1)^2.  \label{def:F}
\end{align}
\end{thm}

\remark (i) Note that the above upper bound depends on the worker set $S$ and their reliabilities $\wi$ only through the term $F(S)$. 
In fact, according to the proof in the supplementary, the term $F(S)$ corresponds to the expected gap between the voting score of the true label $y_i$ (i.e. $\sum_{i\in\S}  \wflin(\hwi)\I{\zij=y_i}$) and that of the wrong labels, and hence reflects the confidence of the weighted majority voting. Therefore, $F(S)$ represents a score function for the worker set $S$:  if $F(S)$ is large, the weighted majority voting is more likely to give correct prediction. 

(ii) The assumption \eqref{model} used in Theorem~\ref{thm:erBound} implies a ``one-coin" model on the workers labels, where the labels are correct with probability $\wi$, and otherwise make mistakes uniformly among the remaining classes. 
This is a common assumption to make, especially in theoretical works (see e.g.,  \cite{Karger2011nips, ghosh2011moderates, Joglekar2013}). 
It is possible to relax \eqref{model} to a more general ``two-coin" model with arbitrary probability $\P(\zij=l|\yj=k)$, which, however, may lead more complex upper bounds. 
In our empirical study on various real-world datasets, we find that $F(S)$ remains to be an efficient score function for worker section even when the one-coin assumption does not seem to hold. 

Based on \eqref{def:F}, it is natural to select the workers by maximizing the term $F(S)$, that is, 
\begin{align}\label{opt:trueW}
\argmaxS \FS, 	\qquad\quad s.t. \quad |\S|\leq \Kmax,
\end{align}
Unfortunately, $F(S)$ depends on the workers' true reliabilities $\wi$, which is often unknown. 
We instead estimate $F(S)$ based on $\hwi$. The following theorem provides an unbiased estimator. 

\begin{lem}\label{thm:unbias}
Assume $\hwi$ is an unbiased estimator of $\wi$ that satisfies $\E[\hwi] = \wi$, and $\hat{\mathrm{var}}(\hwi)$ is an unbiased estimator of the variance of $\hwi$. Consider 
\begin{eqnarray}\label{def:FhatS}
\FhatS = { \inv{\sqrt{|\S|}} \sumiS \G(\hwi) }, 
\end{eqnarray}
where 
\begin{eqnarray}\label{def:Gacc} 
\G(\hwi) = (\nL\hwi-1)^2 -   \nL^2\hvar(\hwi),
\end{eqnarray}
then $\FhatS$ is an unbiased estimate of $\FS$.
\end{lem}

\remark 
(i) The first term $(\nL\hwi-1)^2$ in \eqref{def:Gacc} shows that 
 the workers with $\hwi$ close to either 1 or 0 should be encouraged; these workers tend to answer the questions either all correctly or all wrongly, and hence are ``strongly informative" in terms of the predicting the true labels.  
Note that these workers with $\hwi=0$ are strongly informative in that they eliminate one possible value (their answer) for the true labels.  
On the other side, more workers also means more noise, so there is a term $\sqrt{|\S|}$ for balancing the signal-noise ratio --- to encourage hiring ``strong" workers instead of only hiring more workers. 

(ii) A simpler estimation of $\FS$ is to directly plug $\hwi$ as $\wi$ into \eqref{def:F}, that is, 
\begin{align}
\label{def:plugin}
\FSplug = \inv{\Lfactor \sqrt{|\S|}} \sum_{i\in\S} (\nL\hwi - 1)^2.
\end{align}
However, this obviously leads to a biased estimator of $\FS$ because of the missing of the variance term in \eqref{def:Gacc}.  
The existence of the variance term is of critical importance: The workers with large uncertainty on the reliabilities should be less favorable compared with these with a more confident estimation. 

Since Lemma \ref{thm:unbias} does not specify $\hwi$ and $\hvar(\hwi)$, the next theorem provides a concrete example of $\FhatS$, based on which a symmetric confidence interval of $\FS$ can be constructed.

\begin{thm}\label{thm:concreteFhatS}
Assume a group of workers are tested with $n$ control questions, and let $\ci$ be the number of correct answers given by worker $i$ on the $n$ control questions. Then an unbiased estimator $\hwi$, with an unbiased estimator of var($\hwi$) can be obtained by 
\begin{eqnarray}\label{def:hwi_hvar}
\hwi = \frac{\ci}{n} 
\connect
\hvar(\hwi)= \frac{\ci(n - \ci)}{n^2(n-1)}. 
\end{eqnarray}
With such $\hwi$ and $\hvar(\hwi)$, the corresponding $\FhatS$ in (\ref{def:FhatS}) is unbiased and the interval $[\FhatS-\alphaMargin, \FhatS+ \alphaMargin]$ covers $\FS$ with probability at least $1-2e^{-2\alpha^2}$ for any $\alpha>0$.
\end{thm}
\remark A discussion about the advantage of the unbiasness of $\FhatS$ and the symmetric confidence interval is deferred to Section \ref{sec:discuss}.
 
Based on the estimation of $\FS$ in Theorem~\ref{thm:concreteFhatS}, 
the optimization problem is rewritten into 
\begin{eqnarray}\label{opt:Fhat}
\argmaxS \FhatS, 	\qquad\quad s.t. \quad |\S|\leq \Kmax, 
\end{eqnarray}
where $\FhatS$ is defined in \eqref{def:FhatS}.
Although this combinatorial problem is neither sub-modular nor super-modular, 
we show it can be exactly solved with a linearithmic time algorithm shown in Algorithm~\ref{alg:WSG}.

\begin{algorithm}[htb]
   \caption{\Workselect algorithm}
   \label{alg:WSG}
\begin{algorithmic}[1]
  \STATE {\bfseries Input:} Worker pool $\workset=\hua{1,2,\ldots, M}$ and estimated reliabilities $\hua{\hwi}_{i\in\workset}$ from $n$ control questions; Number of label classes $\nL$; Cardinality constraint: no more than $\Kmax$ workers per item. 
   \STATE  $\xxi \assign \G(\hwi)$, $\forall i\in\workset$ as in \eqref{def:Gacc}, and sort $\hua{\xxi}_{i\in \workset}$ in descending order so that $\xx_{\order(1)} \geq \xx_{\order(2)}\geq \ldots\geq \xx_{\order(M)}$, where $\order$ is a permutation of $\hua{1,2,\cdots, M}$.
   \STATE $B\assign \min(\Kmax, M)$, \quad $\g_1\assign \xx_{\order(1)}$ \quad and \quad $F_1\assign \g_1 $. 
   \FOR{ $k$ from 2 to $B$ }
   	  \STATE $\g_k\assign \g_{k-1} + \xx_{\order(k)}$  \connect 
   	   $\displaystyle F_k\assign \frac{\g_k}{\sqrt{k}}.$
   \ENDFOR
   \STATE $\displaystyle k^* \assign \min\hua{ \argmax_{1\leq k\leq B} F_k }.$
   
  \STATE {\bfseries Output:} The selected subset of workers $\Sstar\assign \hua{\order(1), \order(2), \cdots, \order(k^*)}$.
\end{algorithmic}
\end{algorithm}

Algorithm~\ref{alg:WSG} progresses by ranking the workers according to $\G(\hwi)$ in a decreasing order, and sequentially evaluates the groups of the top-ranked workers, and then finds the smallest group that has the maximal score $\FhatS$.  
The time complexity of Algorithm \ref{alg:WSG} is $O(|\workset|\log|\workset|)$ and the space complexity is $O(|\workset|)$, where $\workset$ is the whole set of workers. 

The following theorem shows that Algorithm~\ref{alg:WSG} achieves the global optimality of \eqref{opt:Fhat}. 
\begin{thm}\label{thm:optimality}
For any fixed  $\hwiall$. The set $\Sstar$ given by Algorithm~\ref{alg:WSG} is a global optimum of Problem~\eqref{opt:Fhat}, that is, we have $\hat F(\Sstar) \geq \hat F(S)  $ for $\forall S\in \workset$ that satisfies $|S| \leq K$. 
\end{thm}

\remark 
As a generalization, consider the following multiple-objective optimization problem, 
\begin{align*}
\argmaxS ( \FhatS, ~ - |S|), 	\qquad\quad s.t. \quad |\S|\leq \Kmax, 
\end{align*}
which simultaneously maximizes the score $\FhatS$ and minimizes the number $|S|$ of workers actually deployed.  
We can show that $\Sstar$ is in fact a Pareto optimal solution in the sense that there exist no other feasible $S$ that improves over $\FhatS$ in terms of both $\FhatS$ and $|S|$ (details in supplementary).

\section{Experimental results}\label{sec:expResults}
\def \wideR{0.48}

We demonstrate our algorithm using empirical experiments based on both simulated and real-world datasets. 
The empirical results confirm our intuition: 
Selecting a small number of top-ranked workers 
may perform as good as,  or even better than using all the available workers. 
In particular, we show that our worker selection algorithm significantly outperforms the naive procedure that uses all the top $K$ workers. We find that our algorithm tends to select a very small number of workers (less than $10$ in all our experiments), which is very close to the optimal number of the top-ranked workers in practice.

To be specific, we consider the following practical scenario in the experiments: 
(i) Assume there is a worker pool $\Omega$ where each worker has completed a ``qualify exam" with $n$ control questions, which is required by either the platform or a particular task owner.  
(ii) The task owner selects a subset of workers from $\Omega$ using a worker selection algorithm such as Algorithm~\ref{alg:WSG} based on their performance on the qualify exam.  
(iii) The selected workers are distributed to answer the $N$ questions of the main interest. (iv) Label aggregation algorithms such as WMV-linear or EM are applied to predict the final labels of these $N$ items. 

Even though our worker selection algorithm is derived when using WMV-linear, we can still use other label aggregation algorithms such as EM, once the worker set is selected. This gives the following possible combinations of the algorithms that we test:  
 WMV-linear on the top $\Kmax$ workers (\alg{WMV top \Kmax}),  WMV-linear on the worker set $\Sstar$ selected by Algorithm \ref{alg:WSG} (\alg{WMV-lin selected}), and WMV with log ratio weights on the selected worker set $\Sstar$ (\alg{WMV-log selected}), the EM algorithm on randomly selected $\Kmax$ workers (referred as \alg{EM random \Kmax}), EM on the top $\Kmax$ workers ranked (\alg{EM top \Kmax}) and EM on the worker set $\Sstar$ selected (\alg{EM selected}). We also implement the worker selection algorithm based on the plugin estimator in \eqref{def:plugin} (which is the same as Algorithm \ref{alg:WSG}, except replacing $\G(\hwi)$ with $(\nL\hwi-1)^2$), followed with a WMV-linear aggregation algorithm (referred as \alg{WMV-lin plugin}). 
Since the majority voting tends to perform much worse all the other algorithms, we omit it in the plots for clarity.  

In each trial of the algorithms on both the simulated and real-world datasets, 10 items are randomly picked from the collected data as the control items, and the workers' reliabilities $\hwiall$ are estimated based on the accuracy on the control items as \eqref{def:hwi_hvar}. 
In each trial, the number of workers selected by Algorithm \ref{alg:WSG} was stored and the average number of workers was computed for each budget $\Kmax$. 
We terminate all the iterative algorithms at a maximum of 100 iterations. All results are averaged over 100 random trials.

\subsection{Simulated data} 

\begin{figure}[!htb]
\begin{center}
\begin{tabular}{cc}
\includegraphics[width=\wideR\textwidth]{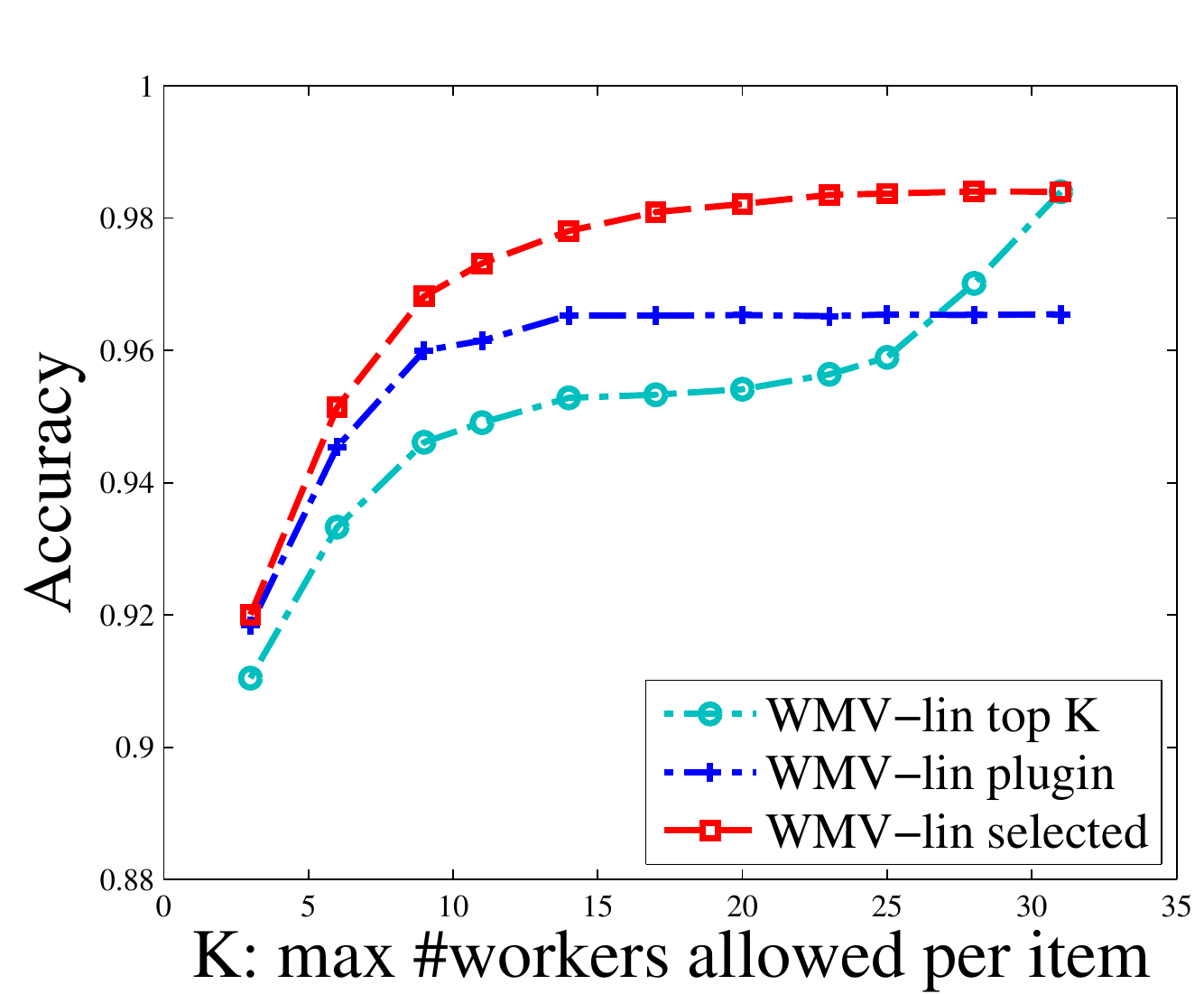}
&
\includegraphics[width=\wideR\textwidth]{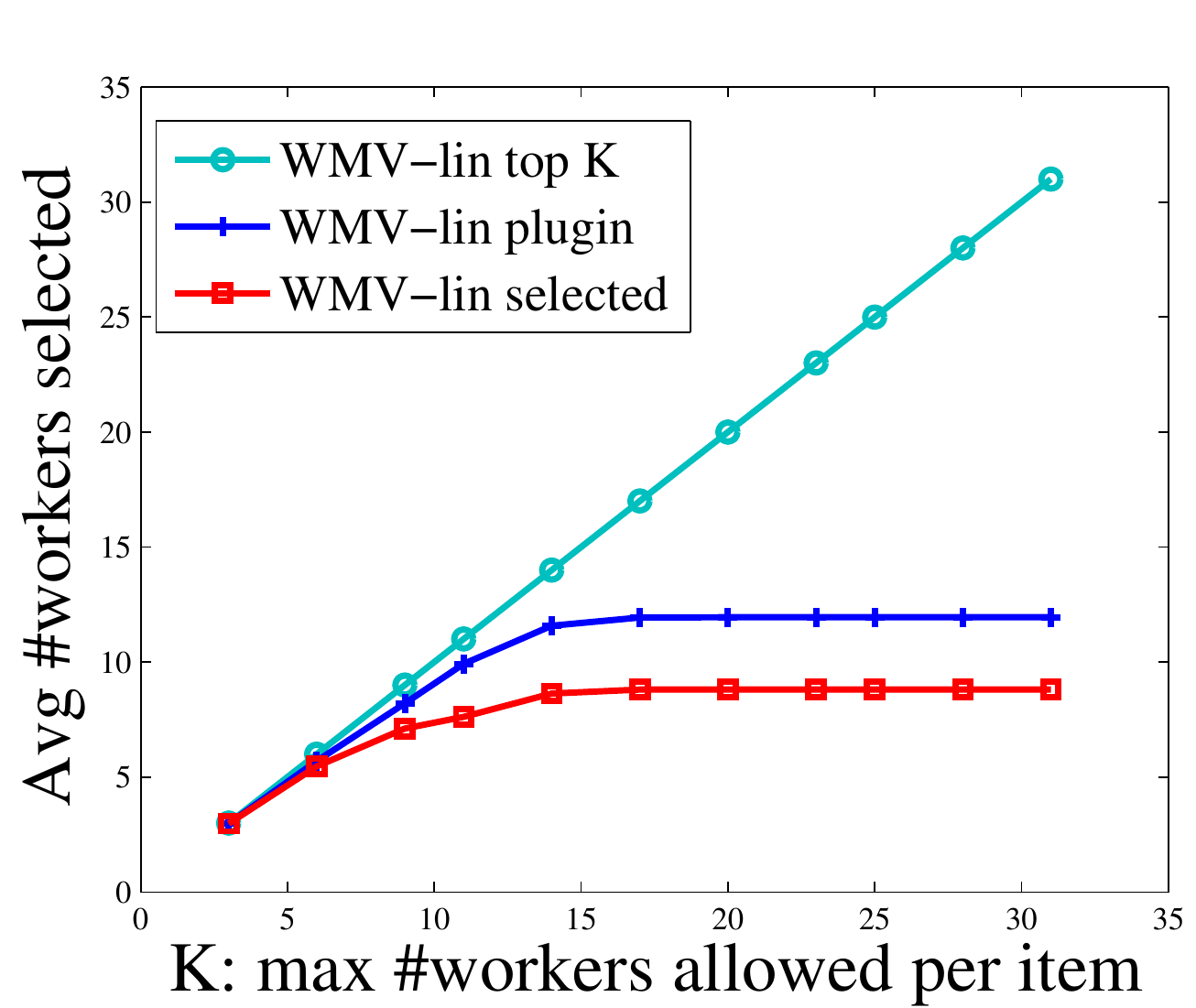}
\\
(a) & (b)
\end{tabular}
\end{center}
\caption{Performance of different worker selection methods on simulated data. WMV-linear aggregation is used in all the cases. We simulated 31 workers and 1000 items with binary labels, and use 10 control questions. The workers' reliabilities are drawn independently from $Beta(2.3, 2)$. (a)  The accuracies when the budget $\Kmax$ varies. 
(b) The actual number of workers used by different worker selection methods when $\Kmax$ increases. 
}
\label{fig:simu}
\end{figure}

We generate the simulated data by drawing 31 workers with reliability $\wi$ from $Beta(2.3, 2)$, and we randomly generated 1000 items with true labels uniformly distributed on $\hua{\pm 1}$. The budget $\Kmax$ varies from 3 to 31. 
Figure \ref{fig:simu}(a) shows the accuracy of  WMV-linear with different worker selection strategies as the budget $\Kmax$ changes. We can see that \alg{WMV-lin selected} dominates the other methods. 
Figure \ref{fig:simu}(b) shows the actual number of workers selected by worker selection algorithm (Algorithm\ref{alg:WSG}). 
WMV-linear based on our selected workers uses a relatively small number of (always $<10$) workers (the red curve in Figure~\ref{fig:simu}(b)), and achieve even better performance than 
\alg{WMV-lin top \Kmax} that uses the entire available budge (the blue line in Figure~\ref{fig:simu}(b)). 
We find that the worker selection algorithm based on the plugin estimator $\FSplug$ tends to select slightly more workers, but achieves slightly worse performance than Algorithm~\ref{alg:WSG} based on the bias-corrected estimator $\FhatS$ (see \alg{WMV-lin select} vs.  \alg{WMV-lin plugin} in Figure~\ref{fig:simu}(a)). This implies the importance of the variance term in (\ref{def:Gacc}), which penalizes the workers with noisy reliability estimation. 
 
\begin{figure*}[tb]
\begin{center}
\begin{tabular}{cc}
\includegraphics[width=\wideR\textwidth]{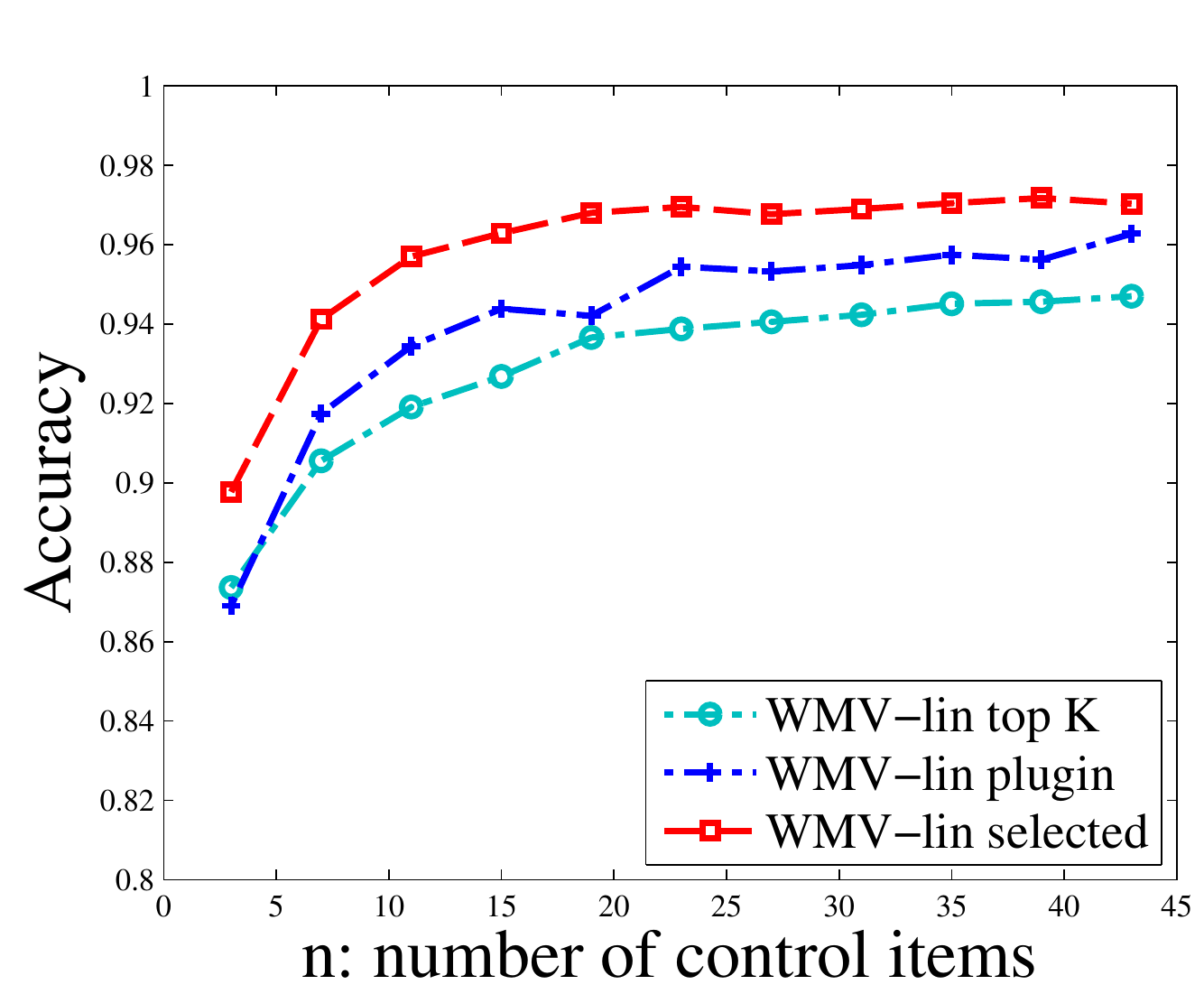}
&
\includegraphics[width=\wideR\textwidth]{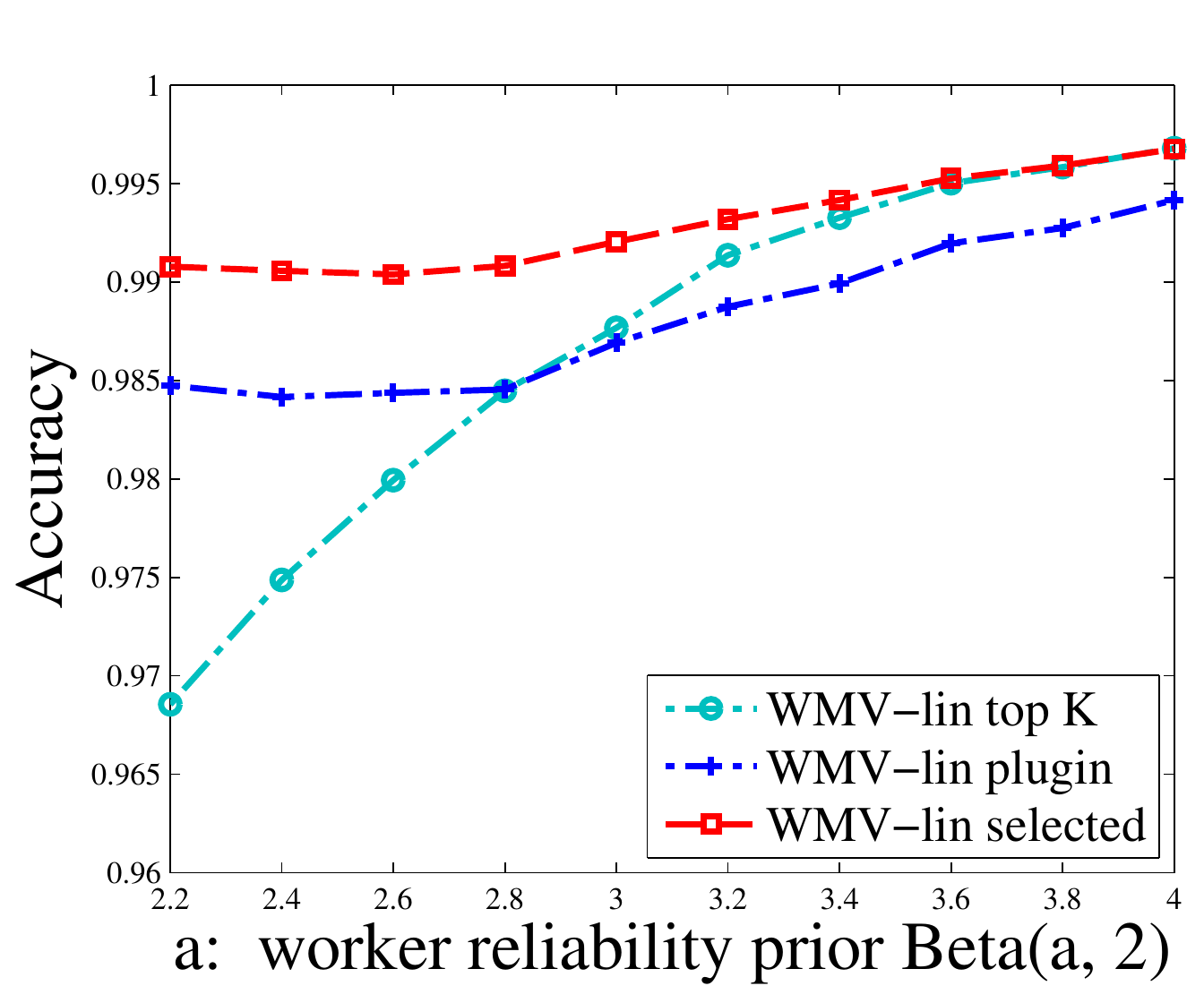}\\
(a) & (b)
\end{tabular}
\end{center}
\caption{Performance of different worker selection methods, (a) when changing the number of control questions $n$, and (b) when changing the parameter $a$ in the reliability prior $Beta(a, 2)$. The budget $\Kmax$ is fixed at 20. We use the WMV-linear aggregation method in all the cases. 
}
\label{fig:changeControl}
\end{figure*}

The number of control questions $n$ controls the variance of the reliability estimation $\hat{w}_i$, and hence influences the results of the worker selection algorithms. 
Figure \ref{fig:changeControl}(a) shows the results when we vary $n$ from 3 to 45, with the budget fixed at 
$\Kmax=20$. 
%
We see that the performance of all the algorithms increases when $n$ increases, because we know more accurate information about the workers' true reliabilities, and can make better decision on both choosing the top $\Kmax$ workers and selecting workers by Algorithm\ref{alg:WSG}. 
In addition, when $n$ increases,  the variance of $\hwi$ decreases and the difference between \alg{WMV-linear selected} and \alg{WMV-linear plugin} decreases.

Figure \ref{fig:changeControl}(b) shows the results when when we vary the prior parameter $a$ where $\wi\sim Beta(a,2)$, fixed $\Kmax=20$ and $n=10$. 
Larger $a$ means the workers are more likely to have high reliabilities (i.e., close to 1). We see from Figure \ref{fig:changeControl}(b) that \alg{WMV-lin top \Kmax} increases as $a$ increases, due to the overall improvement of the reliabilities of the top $\Kmax$ workers. The performances of \alg{WMV-lin selected} and \alg{WMV-lin plugin} improves only slightly, probably because they only select several top workers which is not heavily affected by $a$.

\subsection{Real data}\label{sec:realdata}
We test the different worker selection methods on three real-world datasets: two collected by ourselves from the crowdsourcing platform Clickworkers \footnote{\it http://www.clickworker.com/en}, and one by \citet{Welinder2010nips}
from Amazon Mechanical Turk. 

\emph{Crowd-test dataset:} In this dataset, 31 workers are asked to answer 75 knowledge-based questions from {\it allthetests.com,} which cover topics such as science, math, common knowledge, sports, geography, U.S. history and politics and India. All these questions have 4 options, and we know the all the ground truth beforehand. We required each worker to finish all the questions. 
A typical example of the knowledge-test question is as follows:\\
 \emph{\indent\hspace{1em}  (Question): In what year was the Internet created? \\
\indent\hspace{1em}   (Options): A. 1951; ~ B. 1969; ~C. 1985;~ D. 1993. }  

Figure \ref{fig:readData1}(a) shows the performance of the different methods as the budget $\Kmax$ changes. 
Since EM is widely used in practice, we include the results when using it as the label aggregation algorithm after the workers are selected. 
We find that the performance of \alg{EM Top \Kmax} first increases when $\Kmax$ is small and then decreases when $\Kmax$ is large enough ($\geq 10$ in this case). 
Our worker selection algorithm selects much smaller number of workers, while much better performance, compared to the \alg{top \Kmax} and \alg{random} selection methods.

\begin{figure}[!htb]
\begin{center}
\begin{tabular}{cc}
\includegraphics[width=\wideR\textwidth]{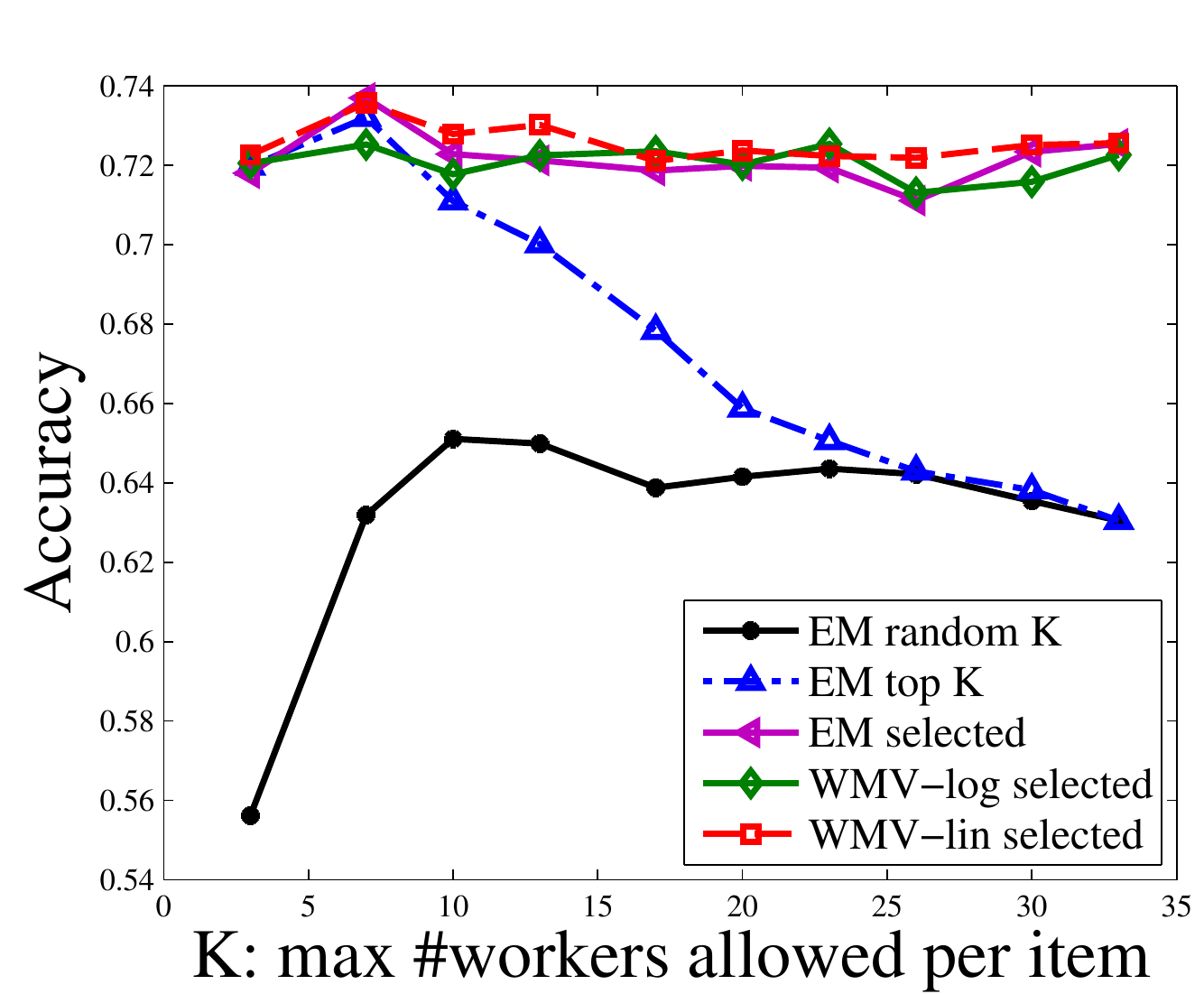}
&
\includegraphics[width=\wideR\textwidth]{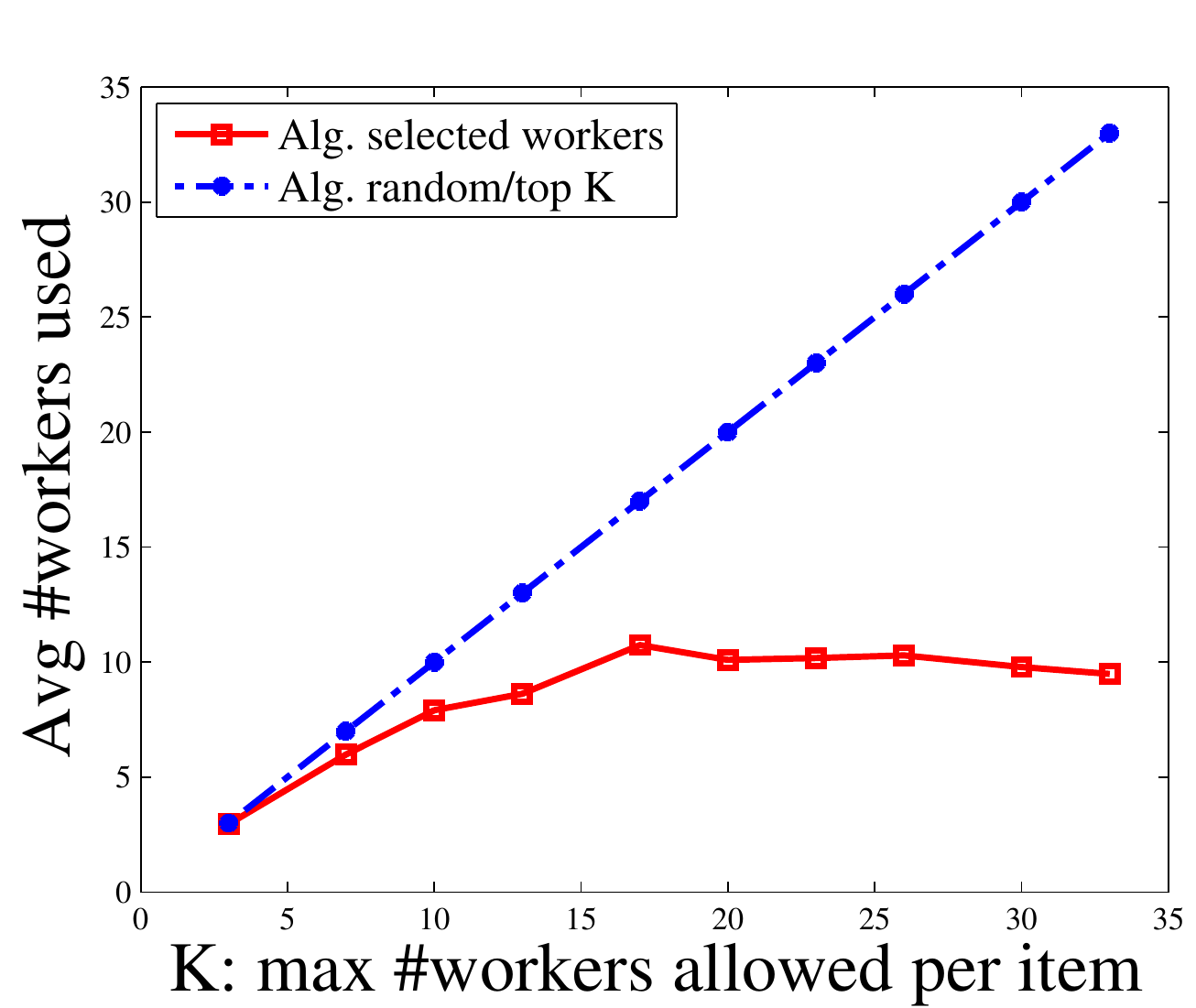}
\\
(a)& (b)
\end{tabular}
\caption{ Crowd-test data. 10 items were randomly selected as control. (a) Performance curve of algorithms with $\Kmax$ increasing. (b) Number of workers the algorithms actually used for each $\Kmax$. } 
\label{fig:readData1}
\end{center}
\end{figure}

\emph{Disambiguity dataset:} 
The task here is to identify which Wikipedia page (within 4 possible options) a given highlighted entity in a sentence actually refers to.  We collected 50 such questions in the technology domain with ground truth available, and hire 35 workers through Clickworkers, each of which is required to complete all the questions. 
A typical example is as follows: 

\emph{
(Question):  ``The Microsoft .Net Framework 4 redistributable package 
install the .NET Framework 
\textbf{\bf runtime} 
and associated files 
that are required to run and develop applications to target the .NET 
Framework 4".  Which Wiki page does ``runtime" refer to? \\
 (Options):
 \\
 A. http://en.wikipedia.org/wiki/Run-time\_system \\
B. http://en.wikipedia.org/wiki/Runtime\_library\\
C. http://en.wikipedia.org/wiki/Run\_time\_(program\\ \indent\hspace{3mm} \_lifecycle\_phase)\\
D. http://en.wikipedia.org/wiki/Run\_Time\_ \\ \indent\hspace{3mm} Infrastructure\_(simulation)
}

\emph{Bluebird dataset:} It is collected by \citet{Welinder2010nips} and is publicly available. In this dataset, 39 workers are asked if a presented image contains Indigo Bunting or Blue GroBeak. 
There are 108 images in total. 

\begin{figure}[!htb]
\begin{center}
\begin{tabular}{cc}
\includegraphics[width=\wideR\textwidth]{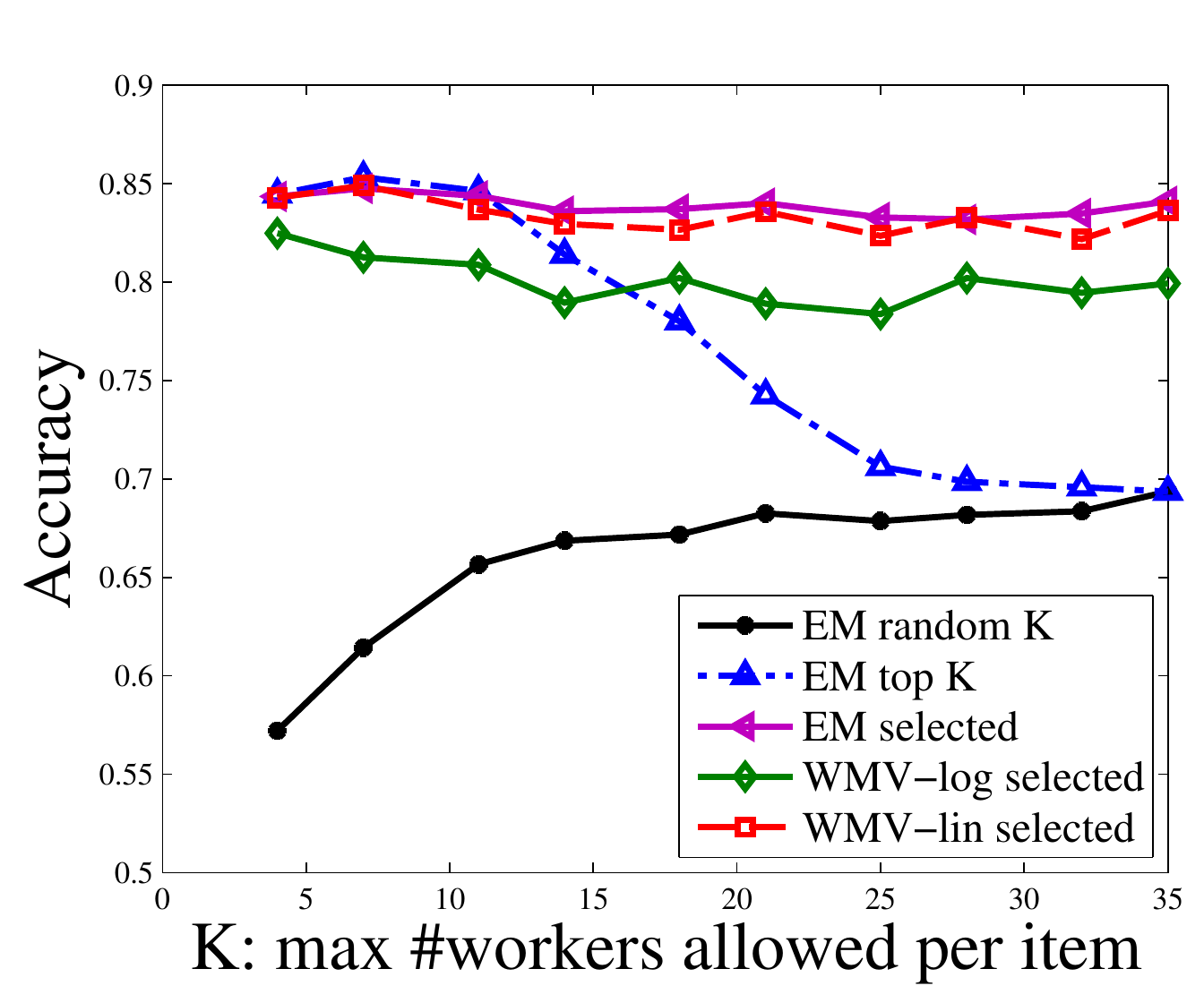}&
\includegraphics[width=\wideR\textwidth]{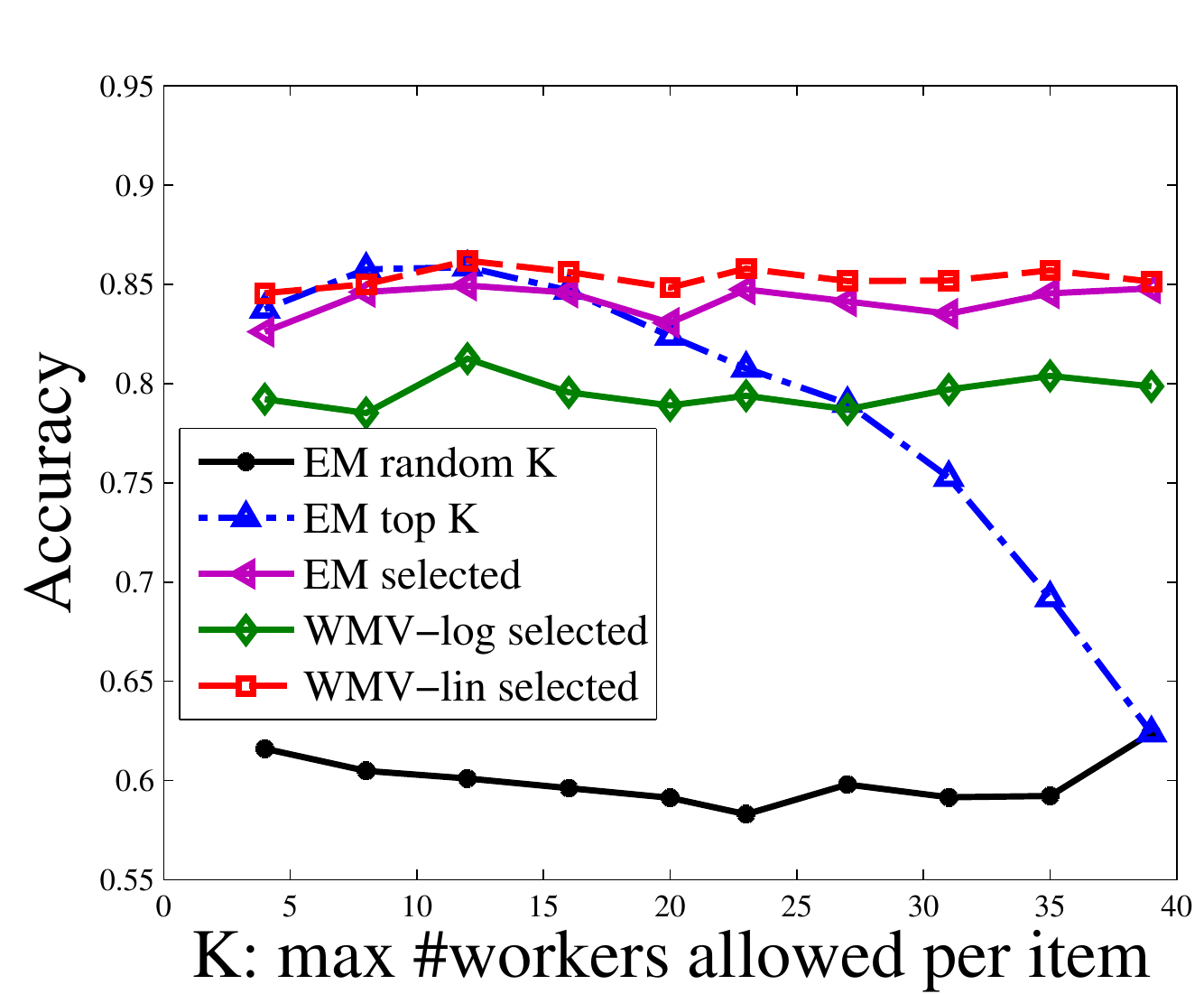}
\\
(a)& (b)
\end{tabular}
\caption{More performance comparison on real-world datasets. (a) The disambiguity dataset: 35 workers and 50 questions in total.  (b) The bluebird dataset: 39 workers and 108 questions in total. 
The settings are the same as that of Figure~\ref{fig:readData1}(a). 
The number of worker actually used (similar to Figure \ref{fig:readData1}(b)) are plotted in supplementary. 
} 
\label{fig:disamb_birds}
\end{center}
\end{figure}

Figure \ref{fig:disamb_birds} (a) and (b) show the performance of the different algorithms on the bluebird and the disambiguation dataset, respectively. The results are similar to the one in Figure \ref{fig:readData1}(a). For the disambiguation dataset, the number of workers selected is usually no more than 6, and the corresponding number for bluebird dataset is 9. 
See the supplementary for the the plots of the number of workers the algorithms actually used (similar to Figure \ref{fig:readData1}(b)) for each $\Kmax$ on these two datasets. 

Note that \alg{WMV-lin selected, WMV-log selected} and \alg{EM selected} are based on the workers selected  by Algorithm \ref{alg:WSG}. They achieve better performance than EM based on the top $\Kmax$ or the random selected workers when $\Kmax$ is large. This shows that aggregation based on inputs from selected workers not only saves budget but also maintains good performance.  



\section{Discussion}\label{sec:discuss}

\emph{What is the advantage of ensuring that $\FhatS$ is an unbiased estimate of $\FS$?} 
The true objective function $\FS$ is unknown, and we can only optimize over a random estimation $\FhatS$. 
If $\FhatS$ is a biased estimator and the bias depends on $\hwiall$, then the optimum solution may be very different from the underlying true solution.  
With the unbiased estimator and the symmetric confidence interval gurantee shown in Lemma~\ref{thm:unbias} and \ref{thm:concreteFhatS}, optimizing $\FhatS$ is equivalent to optimizing a proper confidence bound, because the margin in the confidence interval often does not depend on the workers' reliabilities. The results in Figure \ref{fig:simu} confirm that with the unbiased estimator $\FhatS$, the performance of WMV on the selected workers is better than that with the biased plugin estimator $\Fhat_{\mathrm{plug}}(S)$.

\emph{Why does WMV-linear perform better than WMV with $\wflog$?}
In some of our empirical results (e.g., Figure \ref{fig:disamb_birds}), we find that WMV with log ratio weight is not as good as the one with the linear weight. It is mainly because there is a high chance that some workers get estimated reliability $\hat{w}$ close to 0 or 1 when the number of control questions is small (e.g., $n=10$). Even if we do truncation to prevent a weight $\wflog(\hwi)$ from going to $\infty$, the large weights of some workers may still lead to unstable aggregations. 
However, the performance of WMV with $\wflog$ improves when we use larger $n$ or heavier truncation on $\hwi$. 


\emph{Why does EM with top-$\Kmax$ workers perform poorly as $\Kmax$ increases?} 
Within the given pool of workers, we add increasingly less reliable workers (compared with the workers already selected) as $\Kmax$ increases; 
these less reliable workers may confuse the EM algorithm, causing worse reliability estimation as well as final prediction accuracy.  
This intuition matches with our empirical results in Figure~\ref{fig:readData1} and~\ref{fig:disamb_birds}: 
the performance of EM generally first increases when $\Kmax$ is small (with increasingly more top-quality workers), but then decreases when $\Kmax$ is large (as more less reliable workers are added). 


\section{Conclusion}\label{sec:conclusion}
In this paper, we study the problem of selecting a set of crowd workers to achieve the best accuracy for crowdsourcing labeling tasks.  
We demonstrate that our worker selection algorithm can simultaneously minimize the number of selected workers and minimizing the prediction error rate, achieving the best in terms of both cost and efficiency. 
For future directions, we are interested in developing better selection algorithms based on more advanced label aggregation algorithms such as EM, or more complex probabilistic models.



\nocite{Bachrach_ICML12}
\nocite{Chen2013}
\nocite{Ho2013}
\nocite{Yan_icml11}
\nocite{Yan_icml10}
\nocite{Wauthier2011}

\bibliography{crowdsourcingBibTex}

\begin{thebibliography}{19}
\providecommand{\natexlab}[1]{#1}
\providecommand{\url}[1]{\texttt{#1}}
\expandafter\ifx\csname urlstyle\endcsname\relax
  \providecommand{\doi}[1]{doi: #1}\else
  \providecommand{\doi}{doi: \begingroup \urlstyle{rm}\Url}\fi

\bibitem[Von~Ahn et~al.(2008)Von~Ahn, Maurer, McMillen, Abraham, and
  Blum]{von2008recaptcha}
L.~Von~Ahn, B.~Maurer, C.~McMillen, D.~Abraham, and M.~Blum.
\newblock re{CAPTCHA}: Human-based character recognition via web security
  measures.
\newblock \emph{Science}, 321\penalty0 (5895):\penalty0 1465--1468, 2008.

\bibitem[Liu et~al.(2013)Liu, Ihler, and Steyvers]{Liu2013score}
Q.~Liu, A.~Ihler, and M.~Steyvers.
\newblock Scoring workers in crowdsourcing: How many control questions are
  enough?
\newblock In \emph{NIPS}, 2013.

\bibitem[Dawid and Skene(1979)]{Dawid79jrss}
A.P. Dawid and A.M. Skene.
\newblock Maximum likelihood estimation of observer error-rates using the em
  algorithm.
\newblock \emph{Journal of the Royal Statistical Society.}, 28\penalty0
  (1):\penalty0 20--28, 1979.

\bibitem[Whitehill et~al.(2009)Whitehill, Ruvolo, Wu, Bergsma, and
  Movellan]{Whitehill2009nips}
J.~Whitehill, P.~Ruvolo, T.~Wu, J.~Bergsma, and J.~Movellan.
\newblock Whose vote should count more: Optimal integration of labels from
  labelers of unknown expertise.
\newblock In \emph{NIPS}, 2009.

\bibitem[Karger et~al.(2011)Karger, Oh, and Shah]{Karger2011nips}
D.R. Karger, S.~Oh, and D.~Shah.
\newblock Iterative learning for reliable crowdsourcing systems.
\newblock In \emph{NIPS}, 2011.

\bibitem[Liu et~al.(2012)Liu, Peng, and Ihler]{Liu2012}
Q.~Liu, J.~Peng, and A.~Ihler.
\newblock Variational inference for crowdsourcing.
\newblock In \emph{NIPS}, 2012.

\bibitem[Zhou et~al.(2012)Zhou, Platt, Basu, and Mao]{Zhou2012}
D.~Zhou, J.~Platt, S.~Basu, and Y.~Mao.
\newblock Learning from the wisdom of crowds by minimax entropy.
\newblock In \emph{NIPS}, 2012.

\bibitem[Mannes et~al.(2013)Mannes, Soll, and Larrick]{soll2013}
A.~E. Mannes, J.~B. Soll, and R.~P. Larrick.
\newblock The wisdom of small crowds, 2013.
\newblock URL \url{https://faculty.fuqua.duke.edu/~jsoll/}.

\bibitem[Raykar and Yu(2012)]{Raykar2012eliminate}
V.C. Raykar and S.~Yu.
\newblock Eliminating spammers and ranking annotators for crowdsourced labeling
  tasks.
\newblock \emph{The Journal of Machine Learning Research}, 13:\penalty0
  491--518, 2012.

\bibitem[Joglekar et~al.(2013)Joglekar, Garcia-Molina, and
  Parameswaran]{Joglekar2013}
M.~Joglekar, H.~Garcia-Molina, and A.~Parameswaran.
\newblock Evaluating the crowd with confidence.
\newblock In \emph{SIGKDD}, 2013.

\bibitem[Chen et~al.(2013)Chen, Lin, and Zhou]{Chen2013}
X.~Chen, Q.~Lin, and D.~Zhou.
\newblock Optimistic knowledge gradient policy for optimal budget allocation in
  crowdsourcing.
\newblock In \emph{ICML}, 2013.

\bibitem[Ho et~al.(2013)Ho, Jabbari, and Vaughan]{Ho2013}
C.~Ho, S.~Jabbari, and J.W. Vaughan.
\newblock Adaptive task assignment for crowdsourced classification.
\newblock In \emph{ICML}, 2013.

\bibitem[Li et~al.(2013)Li, Yu, and Zhou]{Li2013b}
H.W. Li, B.~Yu, and D.~Zhou.
\newblock Error rate bounds in crowdsourcing models.
\newblock \emph{arXiv preprint arXiv:1307.2674}, 2013.

\bibitem[Ghosh et~al.(2011)Ghosh, Kale, and McAfee]{ghosh2011moderates}
A.~Ghosh, S.~Kale, and P.~McAfee.
\newblock Who moderates the moderators?: crowdsourcing abuse detection in
  user-generated content.
\newblock In \emph{ACM conference on Electronic commerce}, pages 167--176. ACM,
  2011.

\bibitem[Welinder et~al.(2010)Welinder, Branson, Belongie, and
  Perona]{Welinder2010nips}
P.~Welinder, S.~Branson, S.~Belongie, and P.~Perona.
\newblock The multidimensional wisdom of crowds.
\newblock In \emph{NIPS}, 2010.

\bibitem[Bachrach et~al.(2012)Bachrach, Graepel, Minka, and
  Guiver]{Bachrach_ICML12}
Y.~Bachrach, T.~Graepel, T.~Minka, and J.~Guiver.
\newblock How to grade a test without knowing the answers --- a {B}ayesian
  graphical model for adaptive crowdsourcing and aptitude testing.
\newblock In \emph{ICML}, 2012.

\bibitem[Yan et~al.(2011)Yan, Rosales, Fung, and Dy]{Yan_icml11}
Y.~Yan, R.~Rosales, G.~Fung, and J.G. Dy.
\newblock Active learning from crowds.
\newblock In \emph{ICML}, 2011.

\bibitem[Yan et~al.(2010)Yan, Rosales, Fung, Schmidt, Hermosillo, Bogoni, Moy,
  and Dy]{Yan_icml10}
Y.~Yan, R.~Rosales, G.~Fung, M.~Schmidt, G.~Hermosillo, L.~Bogoni, L.~Moy, and
  J.~G. Dy.
\newblock Modeling annotator expertise : Learning when everybody knows a bit of
  something.
\newblock In \emph{ICML}, 2010.

\bibitem[Wauthier and Jordan(2011)]{Wauthier2011}
F.L. Wauthier and M.I. Jordan.
\newblock {B}ayesian bias mitigation for crowdsourcing.
\newblock In \emph{NIPS}, 2011.

\end{thebibliography}

\bibliographystyle{unsrtnat}

\newpage


\section*{\hspace{30mm}{\Large Supplementary Material }}

\begin{center}

\emph{\large ``Cheaper and Better: Selecting Good Workers for Crowdsourcing "}
\footnote{The equation numbers in this supplementary continue with the ones in the main paper.}

\end{center}

\section*{Proof of Theorem \ref{thm:erBound}: performance guarantee of WMV-linear}
\begin{proof}
Without loss of generality, we denote by $\pri$ the prevalence of true labels, i.e., $\P(\yj=k)=\pri_k, \forall j\in\N, k\in\Labset$,  where $\P$ denotes the probability measure. Note that even in the scenario that $\yj$ is assumed as fixed instead of random, our analysis and results will still hold with $\prik= \I{\yj=k}$. Furthermore, we assume the group of workers are $\S$ with $|\S|= M$. 

For WMV-linear, the weights $\hua{\hvi}_{i=1}^M$ are independent of the data matrix $\Z$. The associated weighted majority voting is 
$$
\hyj= \argmax_{k\in\Labset } \sumi \hvi\I{\zij=k},
$$ 
where $\hvi= \nL\hwi-1$ and $\E[\hwi] =\wi$. Thus, we have $\E\hvi= \vi= \nL\wi-1$ and $-1\leq \hvi \leq \nL-1$. 

Let 
\begin{eqnarray}\label{def:sjl}
\sjk \defas \sumi \hvi\I{\zij=k} , \quad \forall k\in \Labset, j\in\N
\end{eqnarray}
be the aggregated score of $j$th item that on potential label class $k$. Thus the general \predrule can be written as 
$
\hyj= \argmax_{k\in\Labset} \sjk . 
$

We will frequently discuss condition probability, expectation and variance conditioned on the event $\hua{\yj=k}$. Without introducing ambiguity in the context, we define:
\begin{eqnarray}
&& \Pk{~\cdot~} \defas \P(~\cdot~ | \yj=k) \\
&& \Ek{~\cdot~} \defas \E[~\cdot~ | \yj=k] 
\end{eqnarray}

Note that 
\begin{eqnarray}
\Ek{\sjl}=  \sumi \vi\kua{\wi\I{l=k}+ \kua{\frac{1-\wi}{\nL-1}}\I{l\neq k}}, ~~ \forall l,k \in \Labset. 
\end{eqnarray}

First of all, we expand the error probability of labeling the $j$-th item wrong in terms of the conditional probabilities:
\begin{eqnarray}\label{eqn:phyjDecomp}
\P(\hyj\neq \yj) \= \sumkinL \P(\yj=k)\P(\hyj\neq k | \yj=k) 
= \sumkinL \pri_k \Pk{\hyj\neq k}.
\end{eqnarray}

Our major focus in this proof is to bound the term $\Pk{\hyj\neq k}$. 
Our approach will be based on the fact of the following events relations:
\begin{equation}
\bigcup_\lneqk \hua{\sjl>\sjk} 
\quad\subseteq\quad 
\hua{\hyj\neq k} 
\quad\subseteq\quad 
\bigcup_\lneqk\hua{\sjl\geq \sjk} . 
\end{equation}

We want to provide an upper bound for $\P(\hyj\neq \yj)$.
Note that
\begin{eqnarray}
\Pk{\hyj\neq k} &\leq& \Pk{\bigcup_\lneqk \hua{\sjl\geq \sjk}} 
\leq \sum_\lneqk \Pk{\sjl\geq \sjk}.  \label{ineq:sum_prob_sl_geq_sk}
\end{eqnarray}

With $\sjl$ defined as in (\ref{def:sjl}), and when $l\neq k$, we define
\begin{eqnarray}
&& \xikl \defas \sjl -\sjk= \hvi\kua{\I{\zij=l} - \I{\zij=k}}, \\
&& \Ek{\xikl} =  \E\hvi\cdot\kua{\frac{1-\nL\wi}{\nL-1}}= - \inv{\nL-1}(\nL\wi-1)^2, \\
&& \gapjkl \defas \sumi \Ek{\sjk -\sjl}= -\sumi \Ek{\xikl} = \frac{1}{\nL-1}\sumi(\nL\wi-1)^2.
\end{eqnarray}

We have
\begin{eqnarray}
\Pk{\sjl\geq \sjk} \= \Pk{\sumi\hvi\kua{\I{\zij=l} - \I{\zij=k}} \geq 0} \nonumber \\
\= \Pk{\sumi \xikl - \sumi \Ek{\xikl} \geq - \sumi \Ek{\xikl}}, \quad 
\nonumber \\
\= \Pk{\sumi \xikl - \sumi \Ek{\xikl} \geq \gapjkl} \label{ineq:lowerBoundHub}
\end{eqnarray}

Note that $\hua{\xikl}_{i\in\M}$ are conditional independent when given $\hua{\yj=k}$, and they are bounded given the voting weight $\hua{\hvi}_{i\in\M}$ are bounded. Therefore, we could apply Hoeffding concentration inequality 
to further bound $\Pk{\sjl\geq\sjk}$. 

Apparently, 
$
-1\leq \xikl 
\leq (\nL-1).
$
Note $\gapjkl\geq 0$, by appling Hoeffding inequality to (\ref{ineq:lowerBoundHub}), 
\begin{eqnarray*}
\Pk{\sjl\geq \sjk} &\leq& \Pk{\sumi \xikl - \sumi \Ek{\xikl} \geq \gapjkl} 
\\
&\leq& \exp\kua{-\frac{2\gapjkl^2}{\sumi \braket{(\nL-1)-(-1)}^2 }} 
\\
&\leq& \exp\kua{- \frac{2 \gapjkl^2}{\kua{\nL\sqrt{M}}^2}} \qquad\qquad 
\\
&\leq & e^{-2\tsame^2},
\end{eqnarray*}
where 
$
\tsame= \frac{1}{\nL(\nL-1)\sqrt{M}}\sumi (\nL\wi-1)^2.
$

The right hand side of last ineiquality does not depend on $k, l$ or $i$, straightforwardly, 
\begin{eqnarray}
\Pk{\hyj\neq k} &\leq& \sum_\lneqk \Pk{\sjl\geq \sjk} 
\quad\leq\quad  (\nL-1)e^{-2\tsame^2}.
\end{eqnarray}
Furthermore, we have
\begin{eqnarray}
\P(\hyj\neq \yj) \= \sum_{k\in\Labset}\pri_k\Pk{\hyj\neq k} \nn\\
&\leq& (\nL-1)e^{-2\tsame^2}\kua{\sum_{k\in\Labset} \pri_k} \nn\\
\= e^{-2\tsame^2+\ln(\nL-1)} 
\label{ineq:Hoeffding1}
\end{eqnarray}

The bound $e^{-2\tsame^2+\ln(\nL-1)} $  does not depend on $j$, thus it is also a valid bound for the mean error rate.  That is to say
\begin{eqnarray*}
\MER &\leq& \inv{N}\sumj e^{-2\tsame^2+\ln(\nL-1)}  ~~=~~ e^{-2\tsame^2+\ln(\nL-1)} .
\end{eqnarray*}
Note that $\tsame= \frac{\FS}{\nL(\nL-1)}$, thus we have proved the desired result.

\end{proof}

\section*{Proof of Lemma \ref{thm:unbias}: unbiasness of $\FhatS$}

\begin{proof}
Assume $|\S|=k$. Let $\FbS$ be the pluggin estimator for $\FS$, i.e., 
$$
\FbS = \inv{\Lfactor \sqrt{k}}\sumiS (\nL\hwi-1)^2.
$$
First we show that $\FbS$ is a biased estimate of $\FS$. 
\begin{eqnarray*}
\E[\FbS] \= \inv{\Lfactor\sqrt{k}}\sumiS \kua{ \nL^2\E[\hwi^2] - 2\nL\E[\hwi] + 1} \\
\=\inv{\Lfactor\sqrt{k}}\sumiS \kua{\nL^2\kua{ \vari + \wi^2} - 2\nL\wi +1 } \\
\=\inv{\Lfactor\sqrt{k}}\sumiS \kua{ (\nL\wi-1)^2 + \nL^2\vari } \\
\= \FS + \inv{\Lfactor\sqrt{k}}\sumiS \nL^2\vari.
\end{eqnarray*}
Note that $\E\hwi=\wi$ and $\E[\hvari]=\vari$, thus we can move terms around to construct an unbaised estimate of $\FS$ based on $\FbS$:
\begin{eqnarray*}
\FS = \E\braket{ { \FbS - \inv{\Lfactor\sqrt{k}}\sumiS \nL^2\hvari} } , 
\end{eqnarray*}
which leads to a unbaised estimate of $\FS$ as follows.
\begin{eqnarray*}
\FhatS \=  \FbS - \inv{\Lfactor\sqrt{k}}\sumiS \nL^2 \hvari \\
\= \frac{1}{\Lfactor\sqrt{k}} \sumiS  \kua{ (\nL\hwi-1)^2 - \nL^2\hvari },
\end{eqnarray*}
which is the same form as (\ref{def:FhatS}) and $\E[\FhatS]= \FS$.

\end{proof}

\section*{Proof of Theorem \ref{thm:concreteFhatS}: symmetric confidence interval}

\begin{proof}
Similar to the proof of Lemma \ref{thm:unbias}, we assume $|\S|=k$. 
With $\hwi$ and $\hvari$ defined as in (\ref{def:hwi_hvar}), it is straightforwardly to show that $\E\hwi=\wi$ and $\E[\hvari]= \frac{\wi(1-\wi)}{n} = \vari$, 
then by Lemma \ref{thm:unbias} the corresponding unbaised estimator of $\FS$ is 
\begin{eqnarray*}
\FhatS \= \inv{\sqrt{k}}\sumiS \braket{(\nL\hwi-1)^2 - \frac{\nL^2\hwi(1-\hwi)}{n-1} } \\
\= \factLnk \sumiS \braket{ \squareTerm -\lamLn} ,
\end{eqnarray*}
therefore $\G$ has two equivalent forms:
\begin{eqnarray}
&& \G= (\nL\hwi-1)^2 - \frac{\nL^2\hwi(1-\hwi)}{n-1} \label{supp:G_intuitive}\\
\text{~~and\qquad } &&\G= \frac{n}{n-1} \braket{ \squareTerm - \lamLn } \label{supp:G_long}
\end{eqnarray}

Next, we prove the confidence interval of $\FS$ based on the form (\ref{supp:G_long}) of $\G$. We define random variables $\hua{\Xi}_{i\in\S}$ as
$$
\Xi \defas \squareTerm - \lambda ,
$$
where $\lambda= \lamLn$. Then 
$$
\FhatS= \factLnk\sumiS \Xi.
$$
Note that $\hua{\Xi}_{i\in\S}$ are a collection of indepdent random variables,  $ -\lambda\leq \Xi \leq \kua{\nL-1-\frac{\nL-2}{2n}}^2 - \lambda$, and $\E\FhatS= \FS$. We can apply Hoeffding Inequality to bound the following probability,
\begin{eqnarray}
&& \P\kua{\abs{\FhatS - \FS} \leq \factLnk \cdot \beta} \nn\\
\= \P\kua{\factLnk \abs{ \sumiS\Xi - \E\braket{\sumiS \Xi}} \leq \factLnk\cdot \beta} \nn\\
\= \P\kua{ \abs{ \sumiS (\Xi- \E\Xi) }\leq \beta} \nn\\
&\geq& 1-2\exp\kua{- \frac{2\beta^2}{k(\nL-1-\frac{\nL-2}{2n})^4}} \nn\\
\= 1-2e^{-2\alpha^2},
\end{eqnarray}
where $\alpha= \frac{\beta}{(\nL-1-\frac{\nL-2}{2n})^2 \sqrt{k}} $ and the inequality is due to Hoffding bound. Meanwhile,
\begin{eqnarray}
&& \P\kua{\abs{\FhatS - \FS} \leq \factLnk \cdot \beta} \nn \\
 \= \P\kua{\abs{\FhatS - \FS} \leq \frac{n\kua{\nL-1-\frac{\nL-2}{2n}}^2}{(n-1)\Lfactor}\alpha } \nn \\
&\leq& \P\kua{\abs{\FhatS - \FS} \leq \alphaMargin },
\end{eqnarray}
which implies that $[\FhatS-\alphaMargin, \FhatS+\alphaMargin]$ covers $\FS$ with probability at least $1-2e^{-2\alpha^2}$.

\end{proof}

\section*{Proof of Theorem \ref{thm:optimality}: the global optimality of \workselect algorithm }

\begin{proof}

Let $\xxi= (\nL\hwi-1)^2- \nL^2\hvari$,
 then the optimization problem (\ref{opt:Fhat}) can be written as
\begin{eqnarray}\label{opt:x}
\argmaxS \FhatS \qquad s.t.\quad |\S|\leq \Kmax
\end{eqnarray}
where 
$$
\Gf(\S)= \inv{\sqrt{|S|}}\sumiS \xxi.
$$

Note that $\hwiAll$ are given in these optimization problems, thus we do not treat $\xxi$ as random. The problems (\ref{opt:Fhat}) (i.e., (\ref{opt:x})) are deterministic combinatorial problems. In this proof, we show that the output from Algorithm \ref{alg:WSG} achives the global maximum of problem (\ref{opt:x}). 

The worker selection problem is to select a worker set denote by $\Sstar$ such that $\Gf(\Sstar)\geq \Gf(\S)$ for any set of workers $\S\subseteq \workset$. 
Let $\order$ be a permutation of $\workset=\hua{1,2,\cdots, M}$ such that $\xxOrder{1} \geq \xxOrder{2}\geq \ldots\geq \xxOrder{M}$. 

We want to show that given any globally optimal solution of problem (\ref{opt:x}) $\Sstar$, which has cardinality $|\Sstar|= \kstar$, we have $\Gf(\Sstar)=\Gf(\hua{\order(1), \order(2), \cdots, \order(\kstar)})$. 

To see this, let $\S'= \hua{\order(1), \order(2), \cdots, \order(\kstar)}$, and we assume $\Gf(\Sstar) > \Gf(\S')$. Since the value of function $\Gf(\S)$ only depdends on cardinality of $\S$ and $\hua{\xxi}_{i\in\S}$, the configuration
\footnote{Here, we use \emph{configuration} to denote sets that allow duplicates of values such as $\hua{1,1,1, 2, 3, 3}$.}
of values $\hua{\xxi}_{i\in\Sstar}$ is not equal to $\hua{\xxi}_{i\in\S'}$. This further implies that there exist $i\in\Sstar\backslash\S'$ and $j\in\S'\backslash \Sstar$ such that $\xx_i \neq \xx_j$. Since $i\notin\S'$ and $\S'$ is the top $\kstar$ $x$-values, then $\xx_i < \xx_j$. 
Therefore, if we replace $i$ in $\Sstar$ with $j$ will increase the value of $\Gf$, i.e., $\Gf( (\Sstar\backslash\hua{i}) \cup \hua{j}) > \Gf(\Sstar)$. This contradicts with the fact that $\Gf(\Sstar)$ is global optimum. 
Thus we conclude that $\Gf(\Sstar)=\Gf(\hua{\order(1), \order(2), \cdots, \order(\kstar)})$. 

The analysis above implies that if we know the cardinality of the global optimal solution $\kstar$, then the top $\kstar$ workers in terms of $x$-values will be the global optimum in problem (\ref{opt:x}), although it might not be the unique one. Based on the fact that the cardinality of $\Sstar$ has to be one of the values in $\hua{1, 2, \cdots, \min(\Kmax, M)}$, we can compute the value of $\Gf(\hua{\order(1), \order(2), \cdots, \order(k)})$ with $k$ from 1 to $\min(\Kmax, M)$.  Then the maximum of the yielded $\Gf$ function values has to be a global optimum of problem (\ref{opt:x}), and thus the corresponding worker set is global optimum of problem (\ref{opt:Fhat}). Algorithm \ref{alg:WSG} follows exactly the procedure described above, therefore it output a globally optimal worker set.

\end{proof}

As mentioned in the remark of Theorem \ref{thm:optimality}, we can show that $S^*$ also solves the following multi-objective optimization problem that simultaneously maximizes the score $\FhatS$ and minimizes the number $|S|$ of workers actually deployed.  
\begin{thm}
Consider a multiple-objective optimization problem, 
\begin{align*}
\argmaxS ( \FhatS, ~ - |S|), 	\qquad\quad s.t. \quad |\S|\leq \Kmax, 
\end{align*}
then $\Sstar$ is its Pareto optimal solution. 
\end{thm}
\begin{proof}
By Theorem \ref{thm:optimality}, suppose $\Sstar$ is the global optimum of problem \eqref{opt:Fhat} and $|\Sstar|\leq \Kmax$, then there is no other set $\S$ such that $\S\leq \Kmax$ and $\FhatS > \hat F(\Sstar)$. This implies that within the sets with cardinality no more than $\Kmax$, there is no other set could improve $\FhatS$. Thus $\Sstar$ is \emph{Pareto optimal}\footnote{\tt http://en.wikipedia.org/wiki/Multi-objective\_optimization} according to its definition in the context of multiple objective optimization. 
\end{proof}

\begin{figure}[!htb]
\renewcommand{\figurename}{Figure A.\!\!}
\begin{center}
\begin{tabular}{cc}
\includegraphics[width= 0.5\textwidth]{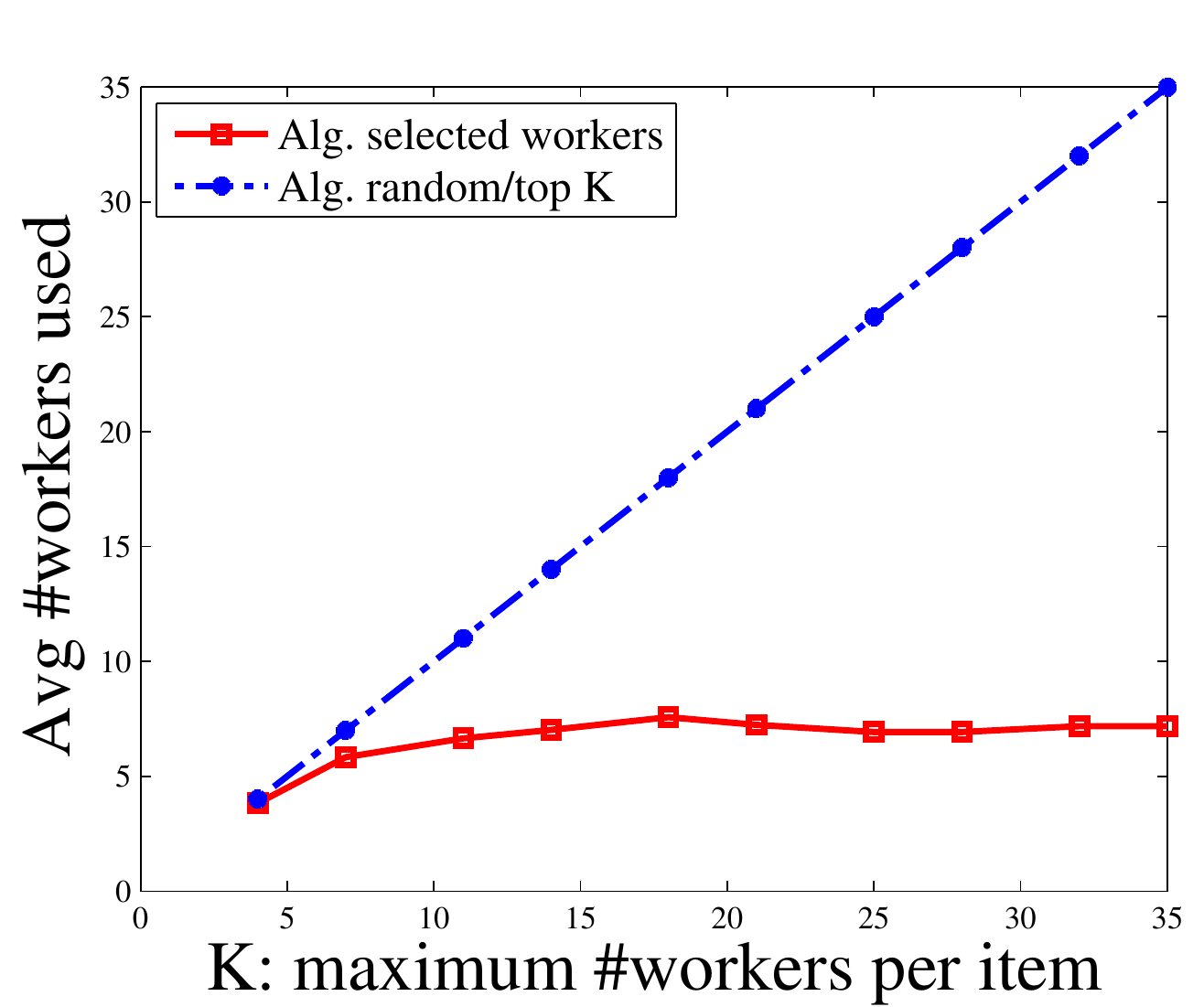}
&
\includegraphics[width=0.5\textwidth]{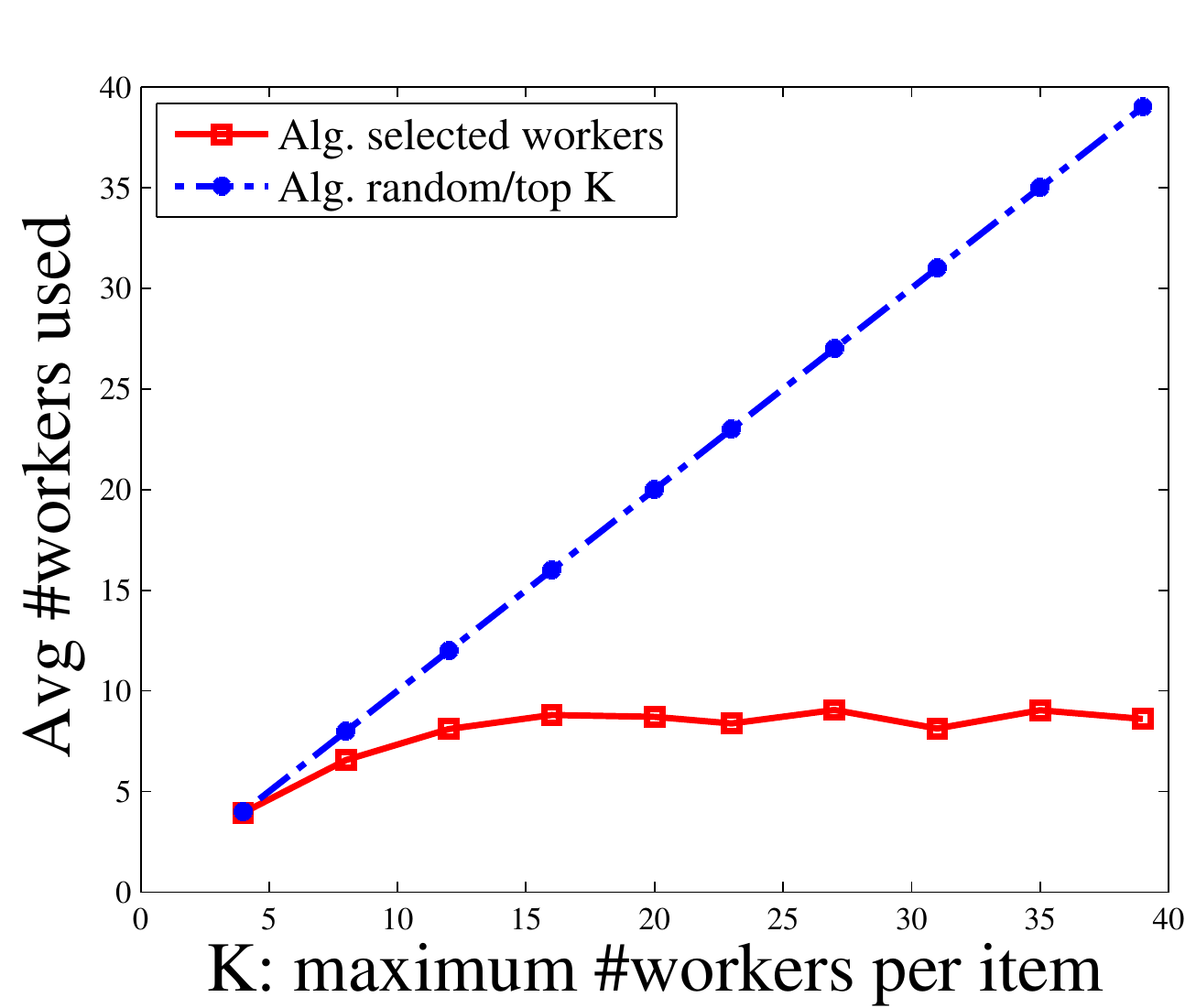}
\\
(a) & (b)
\end{tabular}
\end{center}
\caption{
The number of workers the algorithms actually used for each $\Kmax$ on the two read-world datasets in Figure~\ref{fig:disamb_birds} (Section \ref{sec:realdata} ):
 (a) The disambiguity dataset. (b) The bluebird dataset. 
}
\label{fig:K_selected_twoRealData}
\end{figure}

\end{document}